\documentclass{article}

\usepackage{microtype}
\usepackage{graphicx}
\usepackage{subfigure}
\usepackage{booktabs} 

\usepackage{hyperref}

\usepackage[english]{babel}
\usepackage{blindtext}
 \usepackage[latin1]{inputenc}
 \usepackage{amsthm}
 \usepackage{amsfonts}       
 \usepackage{amsmath}
 \usepackage{cite}
 \theoremstyle{definition}
 \usepackage{subfiles}

\newtheorem{definition}{Definition}[section]
\newtheorem{theorem}{Theorem}[section]
\newtheorem{corollary}{Corollary}[theorem]
\newtheorem{lemma}[theorem]{Lemma}

\newif\ifverbose
\newif\ifcomments

\usepackage{wrapfig,color}
\definecolor{darkblue}{rgb}{0.0, 0.0, 0.8}
\definecolor{darkred}{rgb}{0.8, 0.0, 0.0}
\definecolor{darkgreen}{rgb}{0.0, 0.8, 0.0}

\usepackage[accepted]{icml2020}

\icmltitlerunning{A Note on Over-Smoothing for Graph Neural Networks}

\begin{document}

\twocolumn[
\icmltitle{A Note on Over-Smoothing for Graph Neural Networks}



\icmlsetsymbol{equal}{*}

\begin{icmlauthorlist}
\icmlauthor{Chen Cai}{osu}
\icmlauthor{Yusu Wang}{osu}
\end{icmlauthorlist}

\icmlaffiliation{osu}{Department of Computer Science, Ohio State University, Ohio, USA}

\icmlcorrespondingauthor{Chen Cai}{cai.507@osu.edu}

\icmlkeywords{Machine Learning, ICML}

\vskip 0.3in
]



\printAffiliationsAndNotice{}  

\begin{abstract}
Graph Neural Networks (GNNs) have achieved a lot of success on graph-structured data. However, it is observed that the performance of graph neural networks does not improve as the number of layers increases. This effect, known as over-smoothing \footnote{Strictly speaking, over-smoothing is a misnomer. As we will show, what is decreasing is $tr(X^T \tilde{\Delta} X)$, not the real smoothness $\frac{tr(X^T \tilde{\Delta} X)}{||X||_2^2}$ of graph signal $X$. }, has been analyzed mostly in linear cases. In this paper, we build upon previous results \cite{oono2019graph} to further analyze the over-smoothing effect in the general graph neural network architecture. 
We show when the weight matrix satisfies the conditions determined by the spectrum of augmented normalized Laplacian, the Dirichlet energy of embeddings will converge to zero, resulting in the loss of discriminative power.
Using Dirichlet energy to measure ``expressiveness" of embedding is conceptually clean; it leads to simpler proofs than \cite{oono2019graph} and can handle more non-linearities. 
\end{abstract}

\section{Introduction}
Graph neural networks (GNNs) are a family of neural networks that can learn from graph-structured data. Starting with the success of Graph Convolutional Network (GCN) \cite{kipf2016semi} in achieving state-of-the-art performance on semi-supervised classification, several variants of GNNs have been developed for this task, including GraphSAGE \cite{hamilton2017inductive}, GAT \cite{velivckovic2017graph}, SGC \cite{wu2019simplifying}, CGCNN \cite{xie2018crystal} and GMNN \cite{qu2019gmnn} to name a few most recent ones. See \cite{wu2020comprehensive, zhou2018graph, gurukar2019network} for survey.

However, a key issue with GNNs is their depth limitations. It has been observed that deeply stacking the
layers often results in significantly worse performance for GNNs, such as GCN and GAT. This drop is associated with many factors, including the
vanishing gradients during back-propagation, overfitting due to the increasing number of parameters, as
well as the phenomenon called over-smoothing. \cite{li2018deeper} was the first to call attention to the
over-smoothing problem. Having shown that the graph convolution is a type of Laplacian smoothing,
they proved that after repeatedly applying Laplacian smoothing many times, the features of the nodes
in the (connected) graph would converge to similar values. Later, several others have alluded to the same problem. \cite{li2019deepgcns, zhao2019pairnorm, luan2019break}

The goal of this paper is to extend some analysis of GNN in the ICLR 2020 spotlight paper \cite{oono2019graph} on the expressive power of GNNs for node classification. 
To the best of our knowledge, \cite{oono2019graph} is the first paper extending the analysis of over-smoothing in linear GNNs to the nonlinear ones. However, only ReLU is handled. It is noted by the authors that extension to other non-linearities such as Sigmoid and Leaky ReLU is far from trivial. 

In this paper, we propose a simple technique to analyze the embedding when the number of layers goes to infinity. The analysis is based on tracking the Dirichlet energy of node embeddings across layers. Our contributions are the following: 
\begin{itemize}
\item Using Dirichlet energy to measure expressiveness of embeddings is conceptually clean. Besides being able to recover the results in \cite{oono2019graph}, our analysis can be easily applied to Leaky ReLU. In the special case of regular graphs, our proof can be extended to the most common nonlinearities. The proof is easy to follow and requires only elementary linear algebra. We discuss key differences between our proof and proofs in \cite{oono2019graph} as well as the benefits of introducing Dirichlet energy in Section \ref{diff}.   

\item Second, we perform extensive experiments on a variety of graphs to study the effect of basic edge operations on the Dirichlet energy. We find in many cases dropping edges and increasing the weights of edges (to a high value) can increase the Dirichlet energy.

\end{itemize}

\section{Notation}
Let $\mathbb{N}_{+}$ be the set of positive integers. We define $A \in \mathbb{R}^{N\times N}$ to be the adjacency matrix and $D$ to be the 
degree matrix of graph $G$. Let $\tilde{A}:= A + I_{N}, \tilde{D} := D + I_{N}
$ be the adjacent and degree matrix of graph $G$ augmented with self-
loops. We define the augmented normalized Laplacian of $G$ by $
\tilde{\Delta}:=I_N - \tilde{D}^{-\frac{1}{2} } \tilde{A}\tilde{D}^{-\frac{1}{2}}$ and 
set $P:=I_N - \tilde{\Delta} = \tilde{D}^{-\frac{1}{2} } \tilde{A}\tilde{D}^{-\frac{1}{2}}$. Let $L, C\in \mathbb{N}_{+}$ be the layer and 
channel sizes. 

We define a GCN associated with $G$ by $\mathbf{f} = \mathbf{f}_L \circ ... \circ \mathbf{f}_1$ where $\mathbf{f}_{l}: 
\mathbb{R}^{N \times C_l} \rightarrow \mathbb{R}^{N \times C_{l+1}}$ is defined 
by $\mathbf{f}_l(X) = \text{MLP}_l(PX)$. 
Here $\text{MLP}_{l}(X):=\sigma\left(\cdots \sigma\left(\sigma(X) W_{l 1}\right) W_{l 2} \cdots W_{l H_{l}}\right)$ where $\sigma$ is an element-wise nonlinear function. Note that weight matrices $W_{l \cdot}$ are not necessarily square. We consider the embeddings $X^{(l+1)} := \mathbf{f}_l(X^{(l)})$ with initial value $X^{(0)}$.
We are interested in the asymptotic behavior of the output $X^{(L)}$ of GCN as $L \rightarrow \infty$.

We state the following lemma without a proof.


\begin{lemma}
Eigenvalues of $\tilde{\Delta} \in [0, 2)$.
Eigenvalues of $P = I_N - \tilde{\Delta} \in (-1, 1]$. 
\end{lemma}


\section{Main Result}
The main idea of the proof is to track the Dirichlet energy of node embeddings w.r.t. the (augmented) normalized Laplacian at different layers. With some assumptions on the weight matrix of GCN, we can prove that Dirichlet energy decreases exponentially with respect to the number of layers. Intuitively, the Dirichlet energy of a function measures the ``smoothness" of a function of unit norm, and eigenvectors of the normalized Laplacian are minimizers of the Dirichlet energy.  

\begin{definition}
Dirichlet energy $E(f)$ of scalar function $f \in \mathbb{R}^{N\times 1}$ on the graph $G$ is defined as
\begin{equation*}
E(f) = f^T\tilde{\Delta} f =  \frac{1}{2}\Sigma w_{ij}(\frac{f_i}{\sqrt{1+d_i}} - \frac{f_j}{\sqrt{1+d_j}})^2.
\end{equation*}
For vector field $X_{N\times c} = [x_1, ..., x_N]^T, x_i \in \mathbb{R}^{1\times c}$, Dirichlet energy is defined as
\begin{equation*}
E(X) = tr(X^T\tilde{\Delta} X) =  \frac{1}{2}\Sigma w_{ij}||\frac{x_i}{\sqrt{1+d_i}} - \frac{x_j}{\sqrt{1+d_j}}||_2^2.
\end{equation*}
\end{definition}

 Without loss of generality, each layer of GCN can be represented as $\mathbf{f_l}(X) = \sigma(\underbrace{\sigma(\cdots \sigma(\sigma(}_{H \text { times }}P X) W_{l 1}) W_{l 2} \cdots) W_{l H_{l}})$ 
 Next we will analyze the effects of $P, W_l, \sigma$ on the Dirichlet energy one by one.

\begin{lemma}
\label{lem1}
$E(PX) \leq  (1-\lambda)^2 E(X)$ where $\lambda$ is the  smallest non-zero eigenvalue of $\tilde{\Delta}.$
\end{lemma}

\begin{proof}
Let us denote the eigenvalues of $\tilde{\Delta}$ by $ \lambda_1, \lambda_2, ..., \lambda_N$, and the associated eigenvectors of length 1 by $v_1, ...,  v_n$. Suppose $f = \Sigma c_i v_i$ where $c_i \in \mathbb{R}.$

\begin{equation}
E(f) = f^T \tilde{\Delta} f = f^T \Sigma c_i\lambda_i v_i = \Sigma c_i^2 \lambda_i
\end{equation}
Therefore, 
\begin{equation}
\begin{split}
E(Pf) &= f^T (I_N - \tilde{\Delta})^T\tilde{\Delta}(I_N - \tilde{\Delta})f  \\
&= f^T (I_N - \tilde{\Delta})\tilde{\Delta}(I_N - \tilde{\Delta})f \\
&= \Sigma c_i^2 \lambda_i (1-\lambda_i)^2 \\
&\leq (1-\lambda)^2E(f)
\end{split}
\end{equation}


Extending the above argument from the scaler field to vector field finishes the proof for $E(PX) \leq (1-\lambda)^2 E(X)$.
\end{proof}


\begin{lemma}
\label{lem2}
$E(XW) \leq ||W^T||_2^2 E(X)$
\end{lemma}

\begin{proof} By definition, 
\begin{equation*}
E(XW) = \Sigma_{(i, j) \in E}w_{ij}||\frac{1}{\sqrt{1+d_i}}x_i W- \frac{1}{\sqrt{1+d_j}}x_jW||_2^2
\end{equation*}
where $X_{n\times c} = [x_1, ..., x_n]^T, x_i \in \mathbb{R}^{1\times c},  W\in \mathbb{R}^{c \times c'}.$ Since for each term
\begin{equation}
\begin{split}
||\frac{1}{\sqrt{1+d_i}}x_i W- \frac{1}{\sqrt{1+d_j}}x_jW||_2^2 \leq \\
||\frac{1}{\sqrt{1+d_i}}x_i - \frac{1}{\sqrt{1+d_j}}x_j||_2^2 ||W^T||_2^2,
\end{split}
\end{equation}

we get
\begin{equation}
E(XW) \leq E(W) ||W^T||_2^2
\end{equation}
\end{proof}

\emph{Remark}: Since $||A||_2 = \sigma_{max}(A)$ where $\sigma_{max}(A)$ represents the largest singular value of matrix $A$. Our result in Lemma 2 is essentially the same as the Lemma 2\footnote{$\text {For any } X \in \mathbb{R}^{N \times C}, \text { we have } d_{\mathcal{M}}\left(X W_{l h}\right) \leq s_{l h} d_{\mathcal{M}}(X)$ where $s_{lh}$ is the maximum singular value of $W_{lh}$. } of \cite{oono2019graph}.
Note that our proof can handle weight matrix not only of dimension $d\times d$ but also of dimension $d\times d'$ while the paper \cite{oono2019graph} assumes the embedding dimension to be fixed across layers. See detailed discussion at section \ref{diff}.


\emph{Remark}: The proof itself doesn't leverage the structure of graph. In particular, only the fact of Laplacian is p.s.d matrix is needed in the proof. See an alternative proof in the appendix. This also makes sense because $W$ operates on the graph feature space and should be oblivious to the particular graph structure. 

\begin{lemma}
\label{lem3}
$E(\sigma(X)) \leq E(X)$ when $\sigma$ is ReLU or Leaky-ReLU.
\end{lemma}

\begin{proof}
We first prove it holds for scalar field $f$ and then extend it to vector field $X$. $E(f) = \Sigma_{(i, j)\in E}w_{i,j}(\frac{f_i}{\sqrt{1+d_i}} - \frac{f_j}{\sqrt{1+d_j}} )^2$ where $w_{i,j}\geq 0$.  And
$\forall c_1, c_2 \in \mathbb{R}_{+}, a, b \in \mathbb{R}$
\begin{equation}
\begin{split}
|c_1 a - c_2b| &\geq |\sigma(c_1 a) - \sigma(c_2 b)|  \\
&= |c_1\sigma(a) - c_2\sigma(b)|
\end{split}
\end{equation}

The first inequality holds for all $\sigma$ whose Lipschitz constant is no more than 1, including ReLU, Leaky-ReLU, Tanh, Sigmoid, etc. The second equality holds because for ReLu and Leaky-Relu, $\sigma(cx) = c\sigma(x), \forall c \in \mathbb{R_{+}}, x\in{\mathbb{R}}$.

Therefore, by replacing $c_1, c_2, a, b$ with $\frac{1}{\sqrt{1+d_i}}, \frac{1}{\sqrt{1+d_j}}, f_i, f_j$,   we can see $E(\sigma(f)) \leq E(f)$ holds for ReLU and Leaky-ReLU. Extending the above argument to vector field completes the proof.
%
\end{proof}
\emph{Remark:} For regular graphs, the above conclusion can be extended to more non-linearities such as ReLU, Leaky-ReLU, Tanh, and Sigmoid.

\emph{Remark:} The proof hinges on the simple fact that for ReLU and Leaky-ReLU, $\sigma(c a) = \sigma(c)a$ where $c\in \mathbb{R}_{+}, a\in \mathbb{R}$. For other activation functions, as long as $c_1a = c_2b$ and $c_1 \sigma(a) \neq c_2 \sigma(b)$ (easy to find examples for Sigmoid, Tanh \footnote{Sigmoid: $ \operatorname{ Sigmoid }(x)=\frac{1}{1+e^{-x}}$.
 Tanh: $\operatorname{Tanh}(x) = \frac{e^{2 x}-1}{e^{2 x}+1}$.
}, etc since there are no strong restrictions on $a, b, c_1, c_2$. \footnote{For example, $c_1 = 1, x = 2, c_2 = 2, y = 1$.}), we can not guarantee $E(\sigma(X)) \leq E(X)$.

Combining the above three lemmas, and denote the square of maximum singular value of $W_{lh}^T$ by $s_{lh}$ and set $s_l:= \prod_{h=1}^{H_l} s_{lh}$. Also let $\bar{\lambda} := (1-\lambda)^2$. With those parameters, we arrive at the main theorem.

\begin{theorem}
For any $l \in \mathbb{N}_{+}$, we have $E(\mathbf{f}_l(X)) \leq s_l \bar{\lambda} E(X)$
\end{theorem}
See proof in the appendix \ref{missing_proof}.

\begin{corollary}
Let $s:=\sup _{l \in \mathbb{N}_{+}} s_{l}$. We have $E(X^{(l)}) \leq O((s\bar{\lambda})^l)$. In particular, $E(X^{(l)})$ exponentially converges to 0 when $s\bar{\lambda} < 1.$
\end{corollary}


Our result shares great similarity with the paper \cite{oono2019graph}. The bounds are similar but our result handles more general cases. As noted in \cite{oono2019graph}, eigengap plays an important role. The analysis of Erdos-Renyi graph $G_{N, p}$ (or any other graphs that have large eigengaps) when $\frac{\log N}{N p}=o(1)$ in the paper \cite{oono2019graph}  can also be directly applied to our case.

\begin{figure*}[ht]
\begin{center}
\centerline{\includegraphics[width=\textwidth]{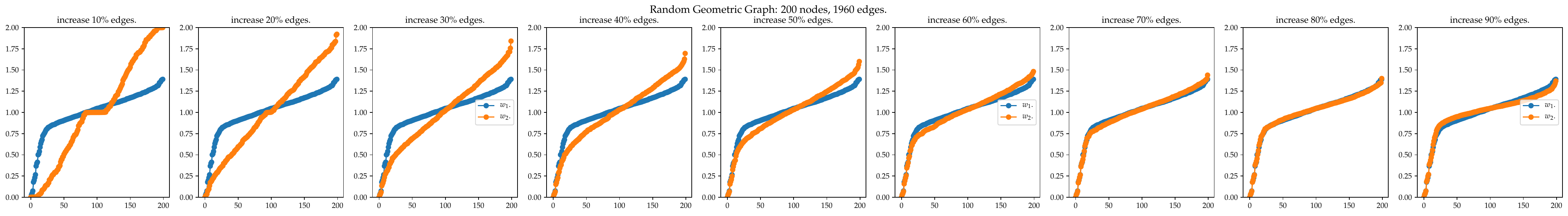}} 
\centerline{\includegraphics[width=\textwidth]{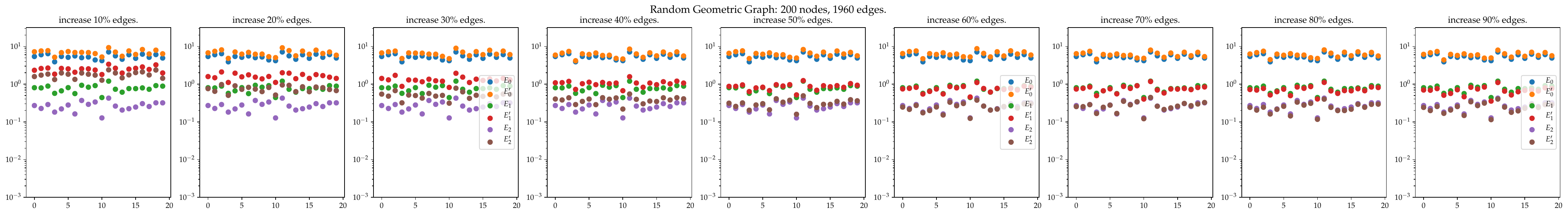}} 
\centerline{\includegraphics[width=\textwidth]{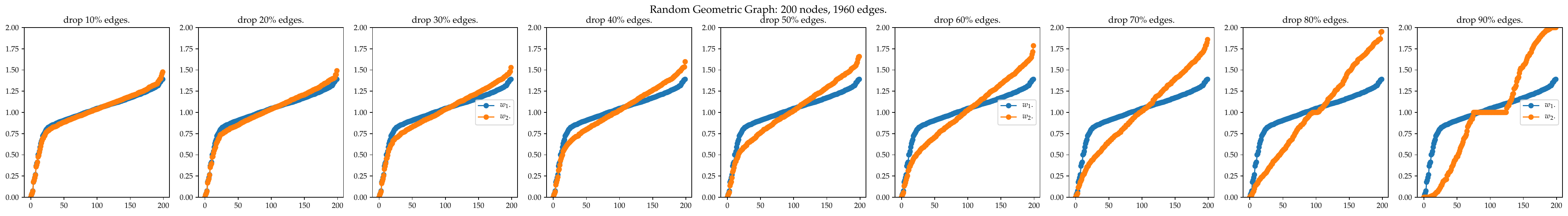}} 
\centerline{\includegraphics[width=\textwidth]{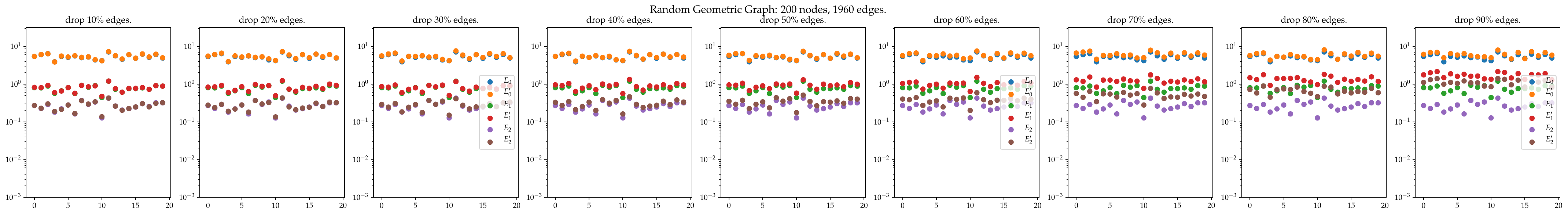}} 
\caption{The effects of dropping edges and increasing edge weights on eigenvalues / Dirichilet energy for Random Geometric graph.}
\end{center}
\vskip -0.2in
\end{figure*}

\section{Key Differences}
\label{sec4}
\label{diff}
The key quantity paper \cite{oono2019graph} looks at is the $d_\mathcal{M}(X)$ where $\mathcal{M}$ is a subspace of  $\mathbb{R}^{N \times C}$, defined as  $\mathcal{M}:=U \otimes \mathbb{R}^{C}=\left\{\sum_{m=1}^{M} e_{m} \otimes w_{m} | w_{m} \in \mathbb{R}^{C}\right\}$, where $(e_m)_{m\in[M]}$ is orthonormal basis of null space $U$ of a normalized graph Laplacian $\tilde{\Delta}$. 
%
The original definition of $d_\mathcal{M}(X)$ is defined for the case of the same embedding dimension across layers. It needs to be modified to handle the case of varying dimensions. 
One way to achieve this is to define $\mathcal{M}=U\otimes \mathbb{R}^C, \mathcal{M'}=U\otimes \mathbb{R}^{C'}$, respectively. The lemma 2 of paper \cite{oono2019graph} then can be modified from $d_{\mathcal{M}}(XW) \leq sd_{\mathcal{M}}(X)$ ($W\in \mathbb{R}^{c\times c}$) to the following:
\begin{equation}
d_{\mathcal{M'}}(XW) \leq sd_{\mathcal{M}}(X)
\end{equation}
where $s$ is the singular value of $W$ \footnote{Here with slight abuse of notation, $W\in \mathbb{R}^{c\times c'}$}. 

As for the nonlinearity, \cite{oono2019graph} mentions that their analysis is limited to graph neural networks with the ReLU activation function because they implicitly use the property that ReLU is a projection onto the cone $\{X  > 0\}$ (see Appendix A, Lemma 3 in \cite{oono2019graph} for details). This fact enables the ReLU function to get along with the nonnegativity of eigenvectors associated with the largest eigenvalues (Perron-Frobenius theorem). Therefore, the authors mentioned that it may not be easy to extend their results to other activation functions such as the sigmoid function or Leaky ReLU. 



In contrast, the proof of Lemma \ref{lem3} becomes trivial once we write out the Dirichlet energy as the sum of multiple terms for each of which the effect of nonlinearity can be easily analyzed. 


\section{Experiments}
To investigate how basic edges operations, removing edges, and increasing edge weight\footnote{In this paper, we only consider the case where the edge weight is increased to very high (from 1 to 10000 in all experiments). }, affect Dirichlet energy and over-smoothing, we perform extensive experiments on both common benchmarks (Cora and CiteSeer) and synthetic graphs. See appendix \ref{exp} for more details on datasets.

In particular, given a graph, we will compute its eigenvalues before and after randomly dropping/increasing weights of a certain percent ($10\% - 90\%$) of edges. This is shown in the first/third column for each figure. In the second/fourth column,  we generate three signals $x, Px(P'x), P^2x(P'^2x)$, where $x=\Sigma_i^T c_i v_i$ where $v_i$ is the first $T$ eigenvectors corresponding to lowest $T$ eigenvalues of normalized Laplacian and $c_i$ is uniform random number between 0 and 1. In other words, $x$ is a mix of lower eigenvectors. We then compute the Dirichlet energy of three signals both for original graph ($E_0, E_1, E_2$) and graph with edges removed/reweighted ($E'_0, E'_1, E'_2$). The same experiments are repeated 20 times and 120 data points are shown in the scatter plot.

We make the following observations. 
\begin{itemize}

\item First, for nearly all graphs and ratios (except for some cases of Cora and CiteSeer), dropping edges increases Dirichlet energy for $x, Px (P'x), P^2x (P'^2x)$. This coincides with the observation in DropEdge \cite{rong2019dropedge} that dropping edges help relive over-smoothing. 

\item Second, in most cases, the effect of increasing the weight of edge (from 1 to 10000) and dropping edges is `` dual" to each other, i.e., increasing weights of a few edges to a very high value is similar to dropping a lot of edges in terms of eigenvalue and Dirichlet energy. Intuitively, we can think of increasing the weight of an edge $u, v$ to infinity as contracting node $u$ and $v$ into a supernode. For the planar graph and its dual graph, edge deletion in one graph corresponds to the contraction in the other graph and vice versa. We hypothesize that randomly increasing the weight of a few edges to a high value will also help to relieve over-smoothing. We leave the systematic verification as future work. 

\end{itemize}

\section{Conclusion}
We provide an alternative proof of graph neural networks exponentially loosing expressive power. Being able to achieve the same bound as the paper \cite{oono2019graph}, our simple proof also handles Leaky ReLU. We also empirically explore the effect of basic edge operations on the Dirichlet energy. 

Some interesting future directions are: 1) The key challenge of analyzing the over-smoothing effect lies in the non-linearity. How to extend our strategy to more general graph learning such as other nonlinearities, normalization strategy \cite{zhao2019pairnorm}, graphs with both node and edge features and attention mechanism \cite{velivckovic2017graph} remains largely open.  2) The assumption on the norm of weight function of GNNs is crucial (may also be too strong) in our proof. Understanding how learning plays a role in resisting the over-smoothing effect is interesting. 3) Preserving Dirichlet energy for combinatorial Laplacian is well studied in the context of graph sparsification. Novel techniques in \cite{spielman2004nearly, spielman2011graph, spielman2011spectral, lee2017sdp} may be applicable.  Also, Dirichlet energy itself is easy to compute and can serve as a useful quantity to monitor during the training of graph networks for practitioners. Finally, analyzing the real over-smoothing effect, i.e., the Rayleigh quotient $\frac{tr(X^T\tilde{\Delta}X)}{||X||_2^2}$, for deep GNNs is still an open and important question. 

\newpage
\appendix

\section{Missing Proof}
\label{missing_proof}
To show that Lemma \ref{lem2} is not using any particular graph structure, we present an alternative proof of Lemma \ref{lem2}, show simply use the generic matrix inequality.
\begin{lemma}
\label{alter_lem2}
$E(XW) \leq ||W^T||_2^2 E(X)$
\end{lemma}

\begin{proof}
Expand $E(XW)$ in matrix form, 
\begin{equation*}
\begin{split}
E(XW) &= tr(W^TX^T\tilde{\Delta} XW) \\
& = tr(X^T\tilde{\Delta} XWW^T)\\
&\leq tr(X^T\tilde{\Delta} X)\sigma_{max}(WW^T) \\
& = E(X) ||W^T||_2^2\\
\end{split}
\end{equation*}
Note $\sigma_{max}$ denotes the largest eigenvalue and $\|A\|_{2}=\sqrt{\lambda_{\max }\left(A^{*} A\right)}=\sigma_{\max }(A)$.
\end{proof}

\begin{theorem}
For any $l \in \mathbb{N}_{+}$, we have $E(\mathbf{f}_l(X)) \leq s_l \bar{\lambda} E(X)$
\end{theorem}

\begin{proof}
By Lemma \ref{lem1}-\ref{lem3}, 
\begin{equation*}
\begin{split}
E(\mathbf{f}_l(X)) &= E( \underbrace{\sigma(\cdots \sigma(\sigma(}_{H \text { times }}P X) W_{l 1}) W_{l 2} \cdots W_{l H_{l}})) \\
& \leq E(\underbrace{\sigma(\cdots \sigma(\sigma(}_{H-1 \text { times }}P X) W_{l 1}) W_{l 2} \cdots) W_{l H_{l}}) \\
& \leq s_{l H_{l}-1} E(\underbrace{\sigma(\cdots \sigma(\sigma(}_{H -1 \text { times }}P X) W_{l 1}) W_{l 2} \cdots) W_{l H_{l}-1}))\\
&...\\
&\leq \left(\prod_{h=1}^{H_{l}} s_{l h}\right) E(P X)\\
&\leq s_l \bar{\lambda} E(PX) \\
&\leq s_l \bar{\lambda} E(X)
\end{split}
\end{equation*}
\end{proof}

\section{Experiments}
\label{exp}
We perform experiments on both synthetic graphs and real graphs benchmarks. The threshold $T$ for the number of lower eigenvectors is set to be 20 for synthetic graphs. For Cora and Citeseer, it is set to be 400 and 600 respectively (due to a large number of nearly zero eigenvalues). The code is available on Github. \footnote{https://github.com/Chen-Cai-OSU/GNN-Over-Smoothing} The details of each graph are listed as follows:
\begin{itemize}
    \item Random graph $G(200, 0.05).$
    
    \item Random geometric graph on the plane. There are 200 nodes uniformly at random in the unit cube. Two nodes are joined by an edge if the distance between the nodes is at most 0.2. 
    
    \item Stochastic Block Model with 2 blocks. It consists of two blocks where each block has 100 nodes. The edge probability within the block is $0.1$ and edge probability between blocks is $0.01$.
    
    \item Stochastic Block Model with 4 blocks. It consists of four blocks where each block has 50 nodes. The edge probability within the block is $0.1, 0.2, 0.3, 0.4$. The edge probability between blocks is $0.08$.
    
    \item Barabasi-Albert Graph. A graph of $n$ nodes is grown by attaching new nodes each with $m$ edges that are preferentially attached to existing nodes with high degree. We set $n, m$ to be $200$ and $4$.
    
    \item Cora is a citation graph where 2708 nodes are documents and 5278 edges are citation links.
    
    \item Citeseer is a citation graph where 3327 nodes are documents and 4552 edges are citation links.

\end{itemize}

\begin{figure*}[ht] 
\vskip 0.2in 
\begin{center} 
 Low Eigenvector Mix + Reweight Edge  
\centerline{\includegraphics[width=\textwidth]{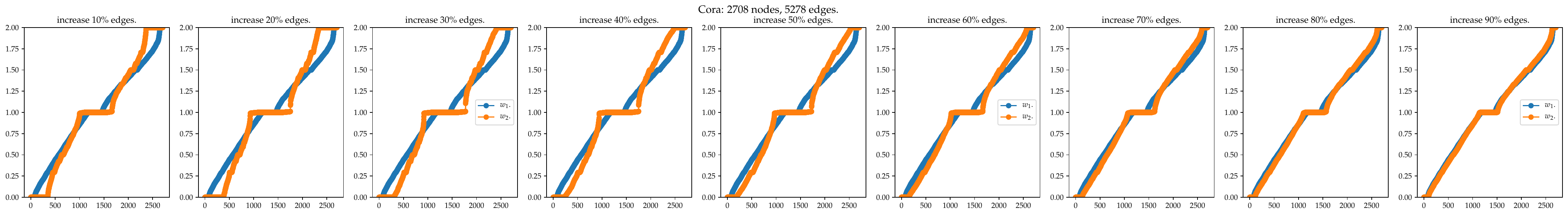}} 
\centerline{\includegraphics[width=\textwidth]{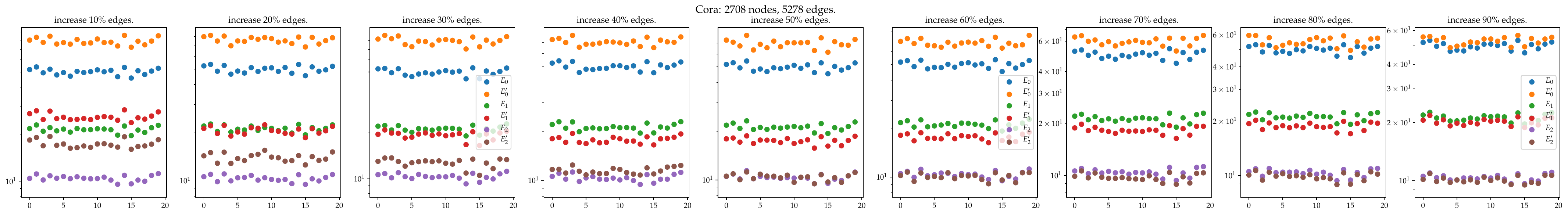}} 
 Low Eigenvector Mix + Drop Edge  
\ 
\centerline{\includegraphics[width=\textwidth]{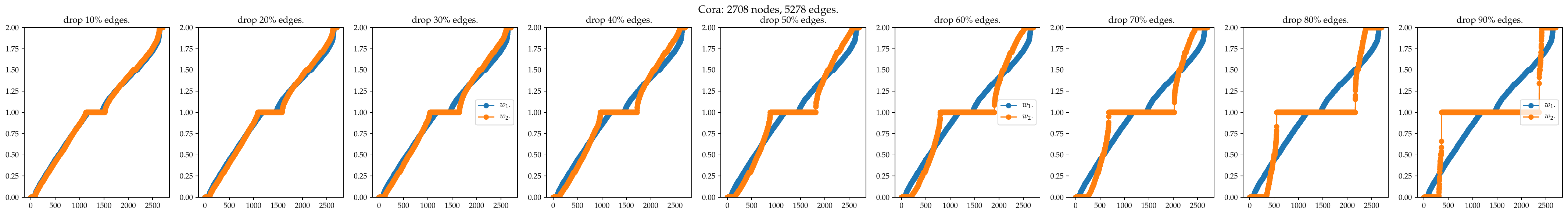}} 
\centerline{\includegraphics[width=\textwidth]{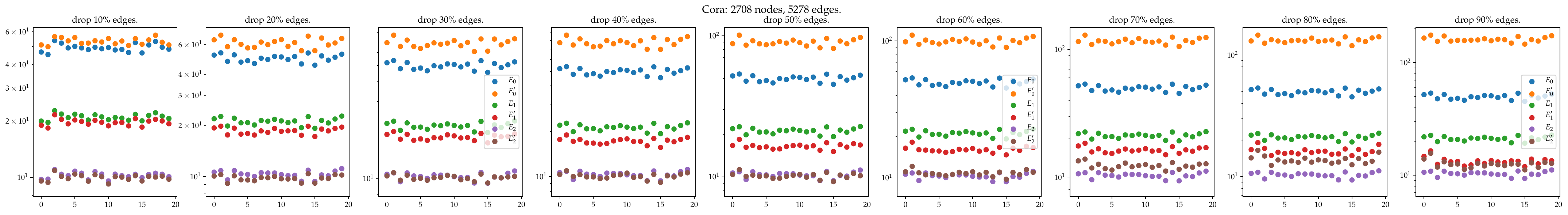}} 
\caption{Cora.} 
\end{center} 
\vskip -0.2in 
\end{figure*} 

\begin{figure*}[ht] 
\vskip 0.2in 
\begin{center} 
 Low Eigenvector Mix + Reweight Edge  
\centerline{\includegraphics[width=\textwidth]{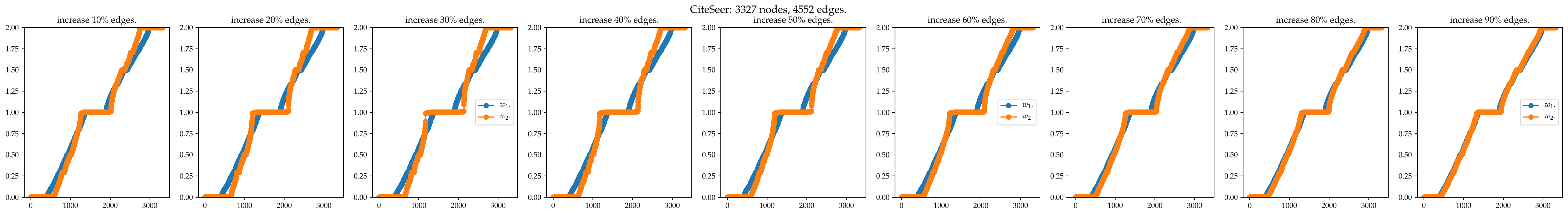}} 
\centerline{\includegraphics[width=\textwidth]{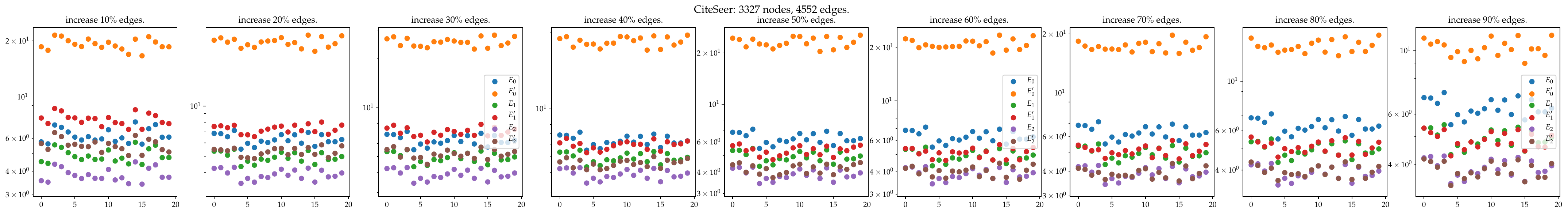}} 
 Low Eigenvector Mix + Drop Edge  
\ 
\centerline{\includegraphics[width=\textwidth]{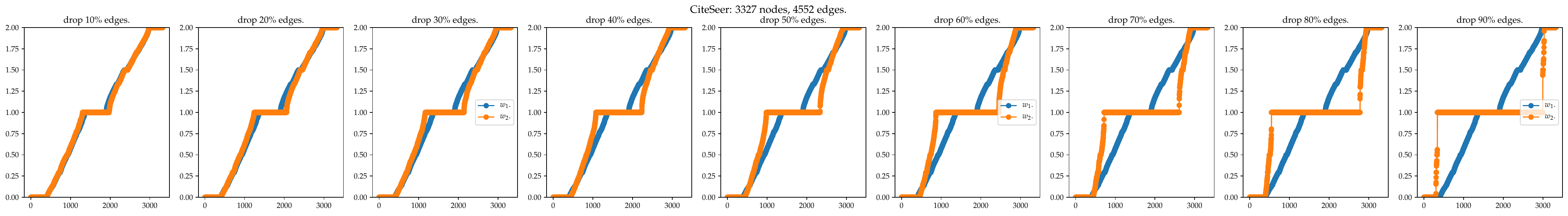}} 
\centerline{\includegraphics[width=\textwidth]{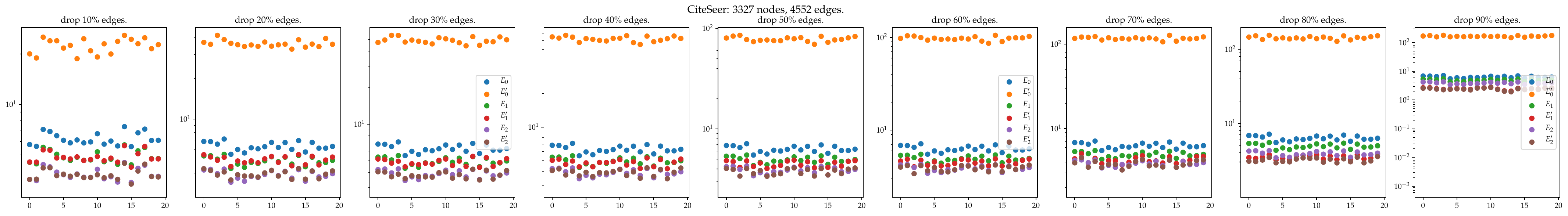}} 
\caption{Citeseer.} 
\end{center} 
\vskip -0.2in 
\end{figure*} 

\begin{figure*}[ht] 
\vskip 0.2in 
\begin{center} 
 Low Eigenvector Mix + Reweight Edge  
\centerline{\includegraphics[width=\textwidth]{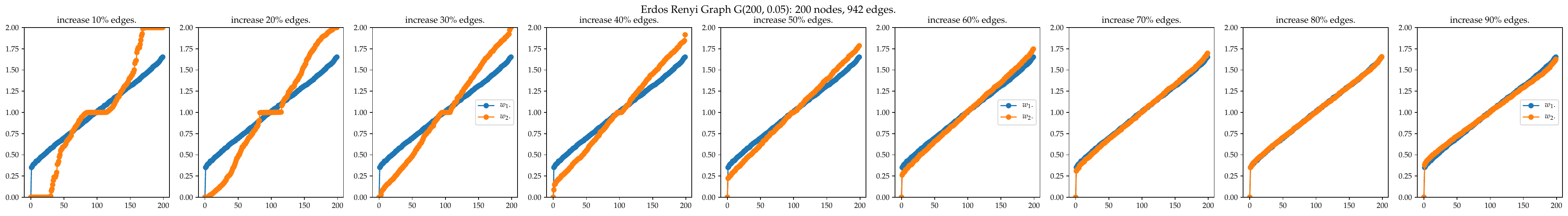}} 
\centerline{\includegraphics[width=\textwidth]{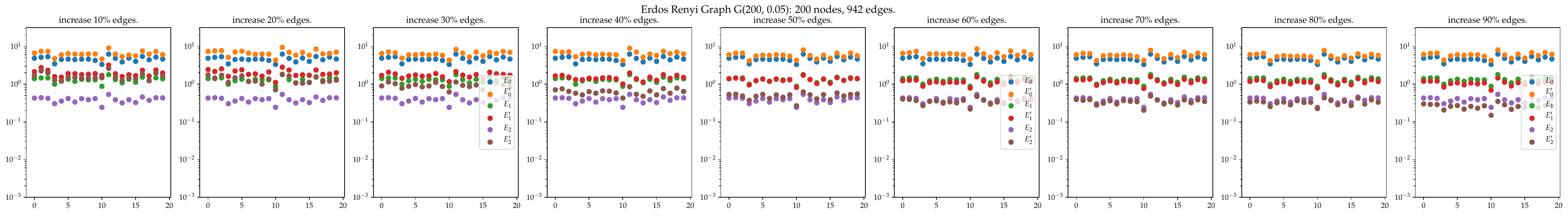}} 
 Low Eigenvector Mix + Drop Edge  
\ 
\centerline{\includegraphics[width=\textwidth]{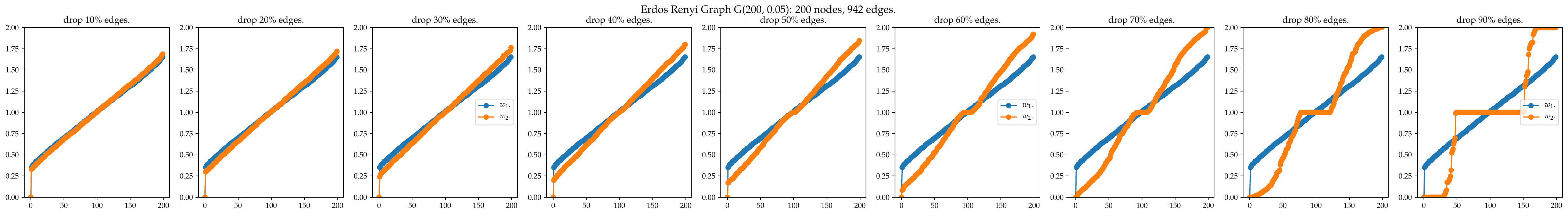}} 
\centerline{\includegraphics[width=\textwidth]{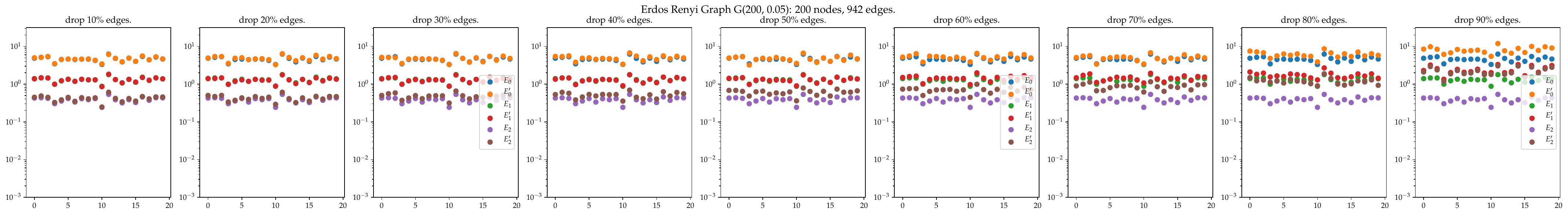}} 
\caption{Random Graph.} 
\end{center} 
\vskip -0.2in 
\end{figure*} 

\begin{figure*}[ht] 
\vskip 0.2in 
\begin{center} 
 Low Eigenvector Mix + Reweight Edge  
\centerline{\includegraphics[width=\textwidth]{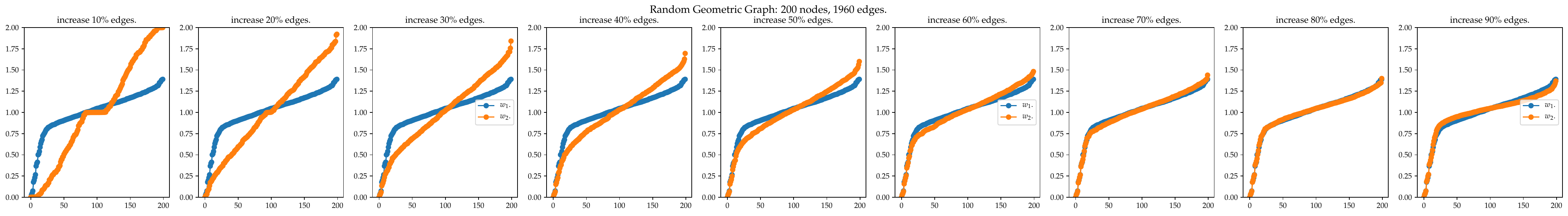}} 
\centerline{\includegraphics[width=\textwidth]{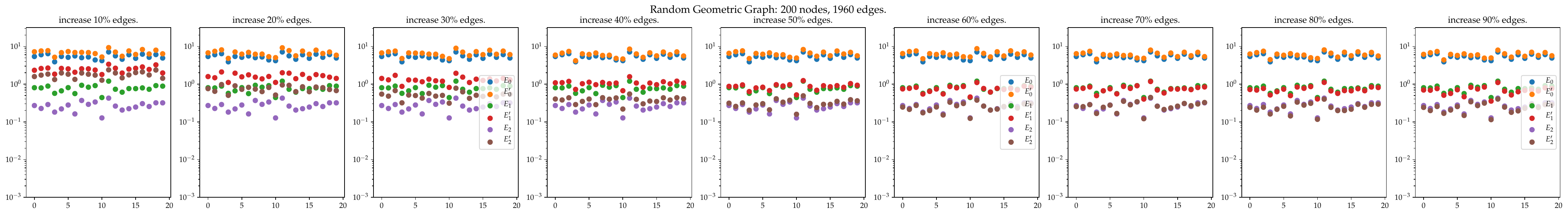}} 
 Low Eigenvector Mix + Drop Edge  
\ 
\centerline{\includegraphics[width=\textwidth]{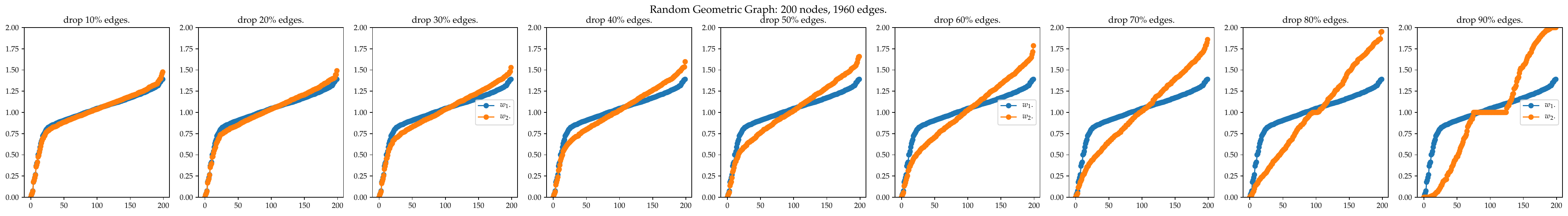}} 
\centerline{\includegraphics[width=\textwidth]{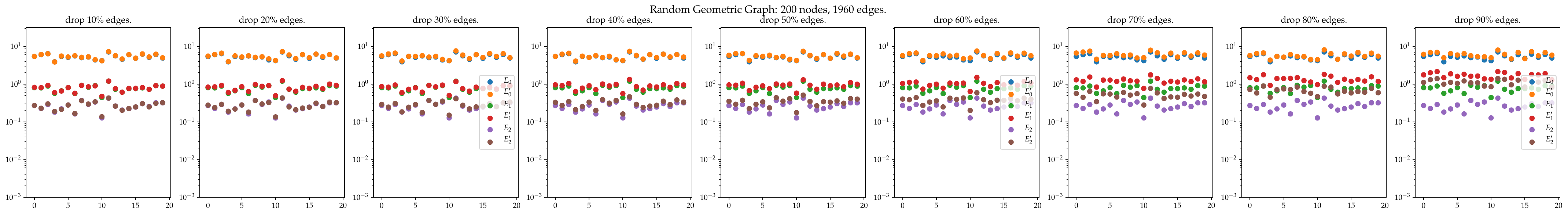}} 
\caption{Random Geometric Graph.} 
\end{center} 
\vskip -0.2in 
\end{figure*} 

\begin{figure*}[ht] 
\vskip 0.2in 
\begin{center} 
 Low Eigenvector Mix + Reweight Edge  
\centerline{\includegraphics[width=\textwidth]{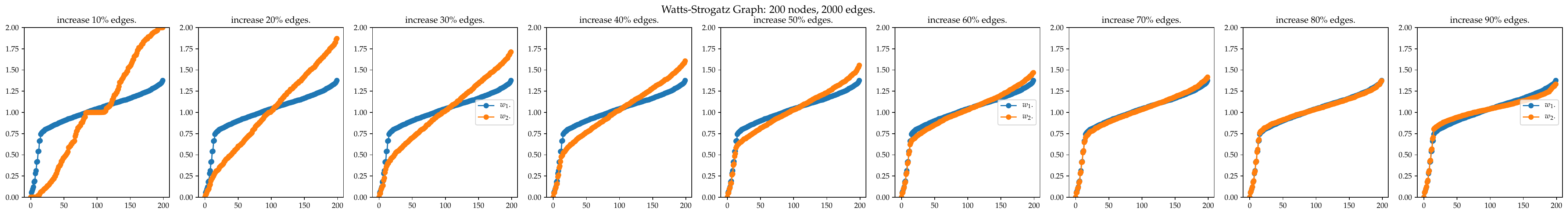}} 
\centerline{\includegraphics[width=\textwidth]{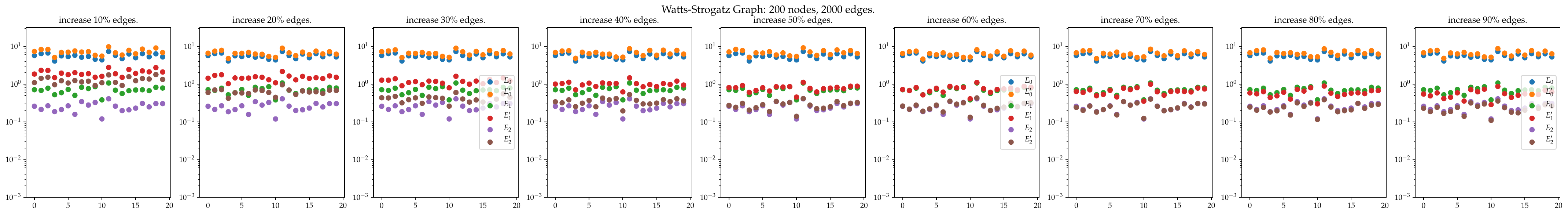}} 
 Low Eigenvector Mix + Drop Edge  
\ 
\centerline{\includegraphics[width=\textwidth]{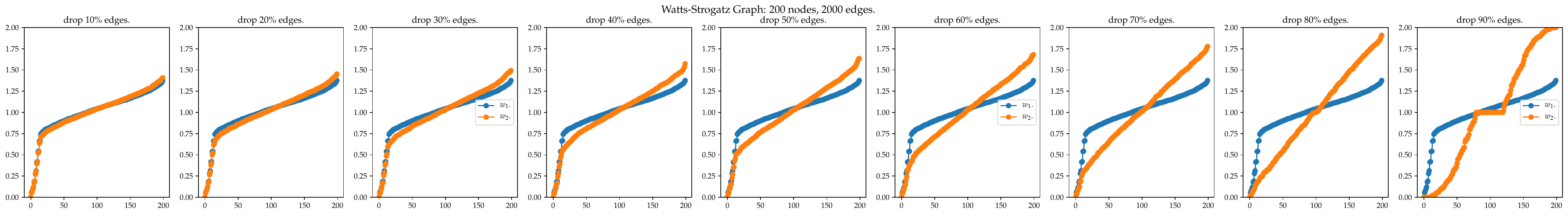}} 
\centerline{\includegraphics[width=\textwidth]{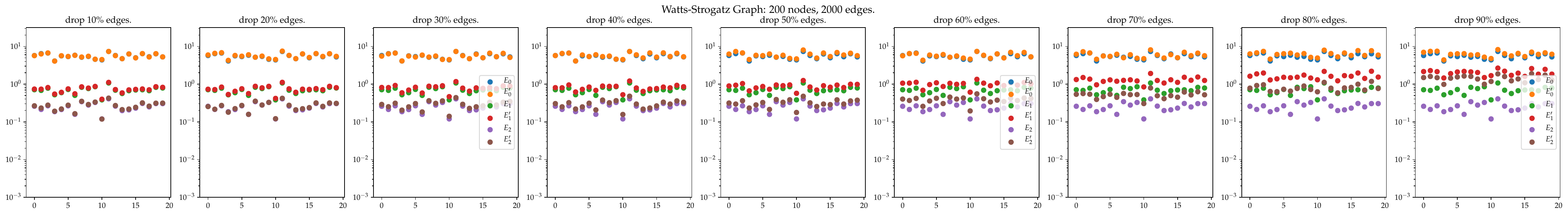}} 
\caption{Watts-Strogatz Graph.} 
\end{center} 
\vskip -0.2in 
\end{figure*} 

\begin{figure*}[ht] 
\vskip 0.2in 
\begin{center} 
 Low Eigenvector Mix + Reweight Edge  
\centerline{\includegraphics[width=\textwidth]{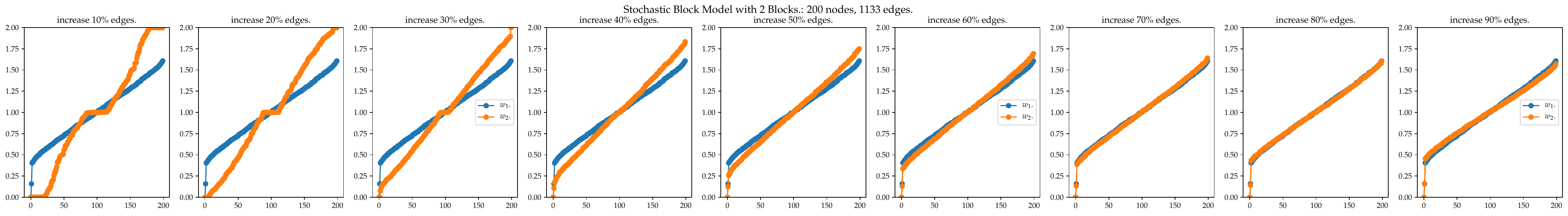}} 
\centerline{\includegraphics[width=\textwidth]{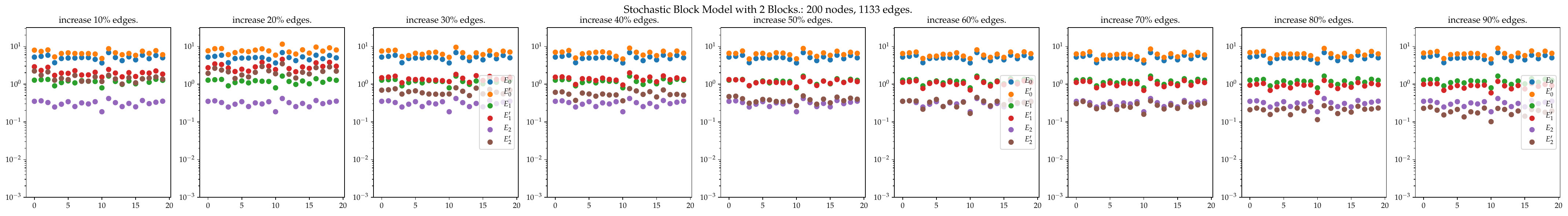}} 
 Low Eigenvector Mix + Drop Edge  
\ 
\centerline{\includegraphics[width=\textwidth]{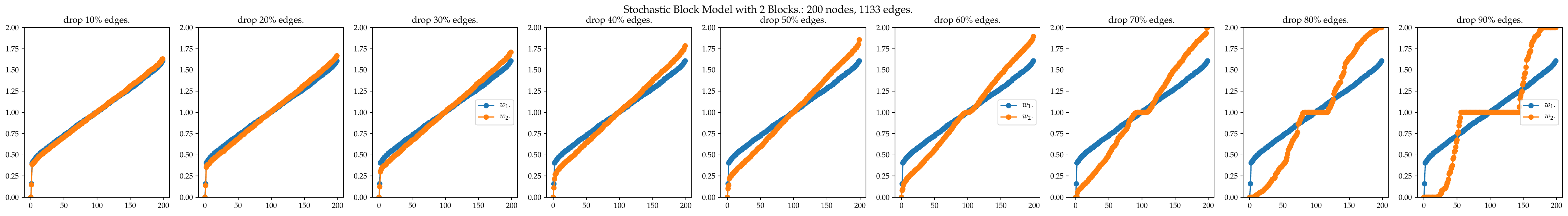}} 
\centerline{\includegraphics[width=\textwidth]{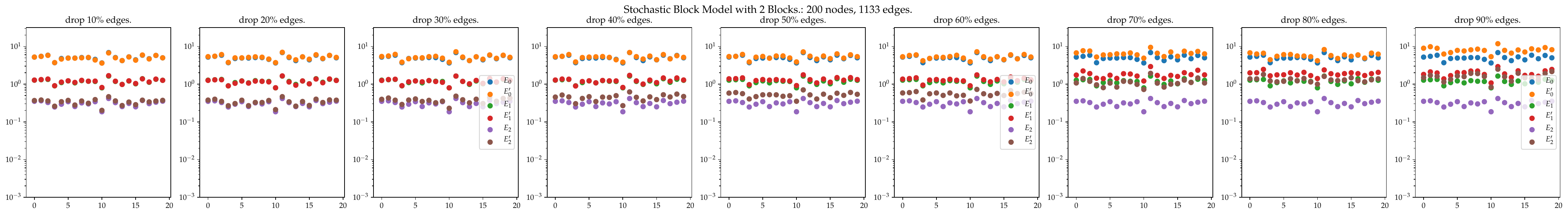}} 
\caption{Stochastic Block Model with 2 Blocks. } 
\end{center} 
\vskip -0.2in 
\end{figure*} 

\begin{figure*}[ht] 
\vskip 0.2in 
\begin{center} 
 Low Eigenvector Mix + Reweight Edge  
\centerline{\includegraphics[width=\textwidth]{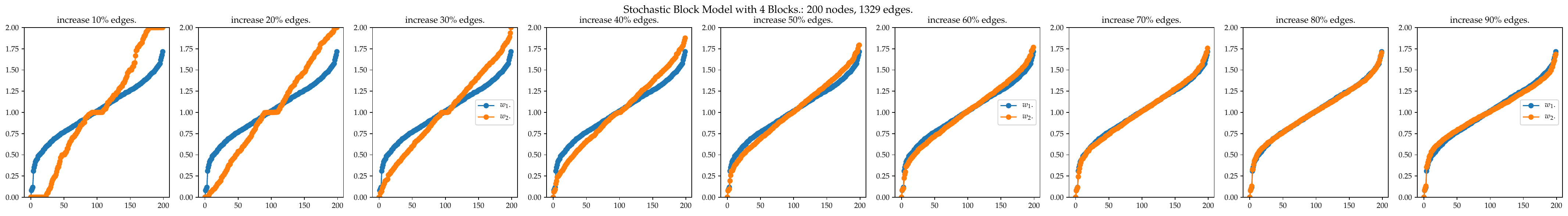}} 
\centerline{\includegraphics[width=\textwidth]{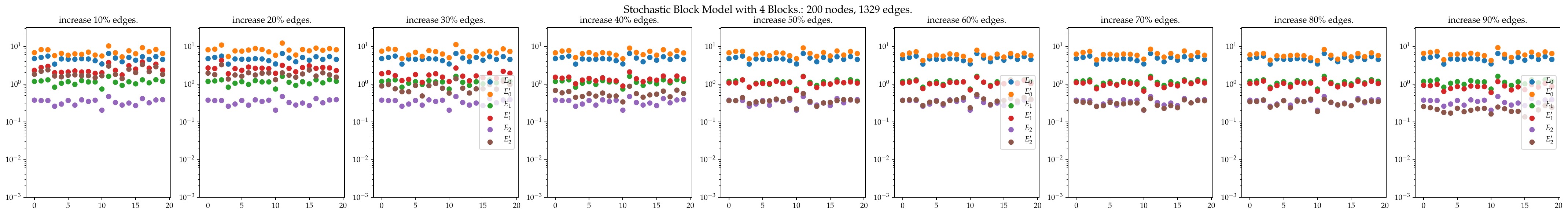}} 
 Low Eigenvector Mix + Drop Edge  
\ 
\centerline{\includegraphics[width=\textwidth]{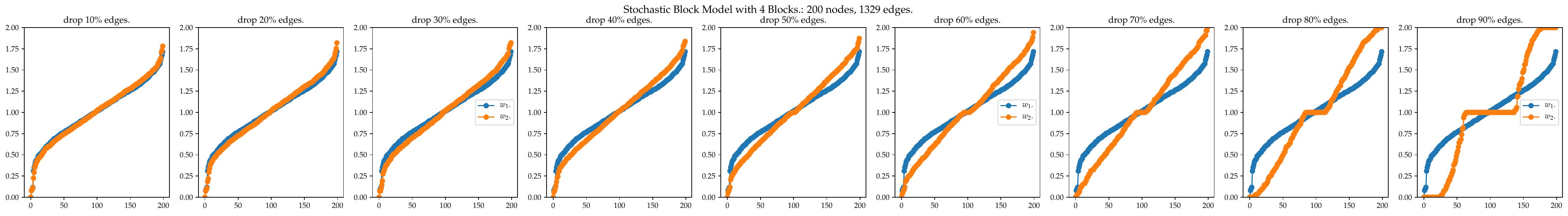}} 
\centerline{\includegraphics[width=\textwidth]{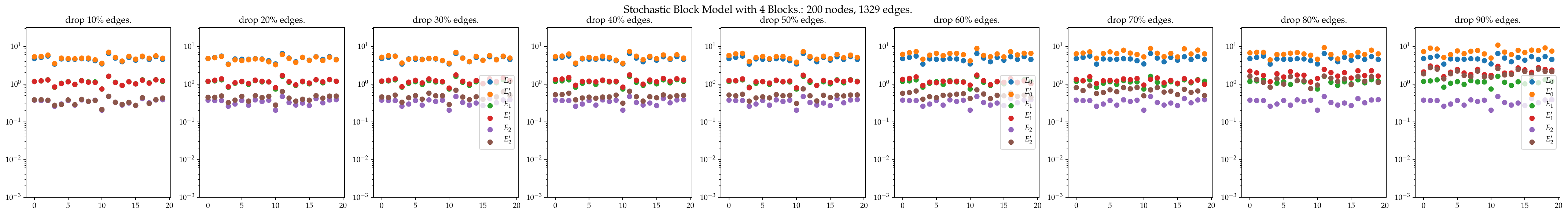}} 
\caption{Stochastic Block Model with 4 Blocks. } 
\end{center} 
\vskip -0.2in 
\end{figure*} 

\begin{figure*}[ht] 
\vskip 0.2in 
\begin{center} 
 Low Eigenvector Mix + Reweight Edge  
\centerline{\includegraphics[width=\textwidth]{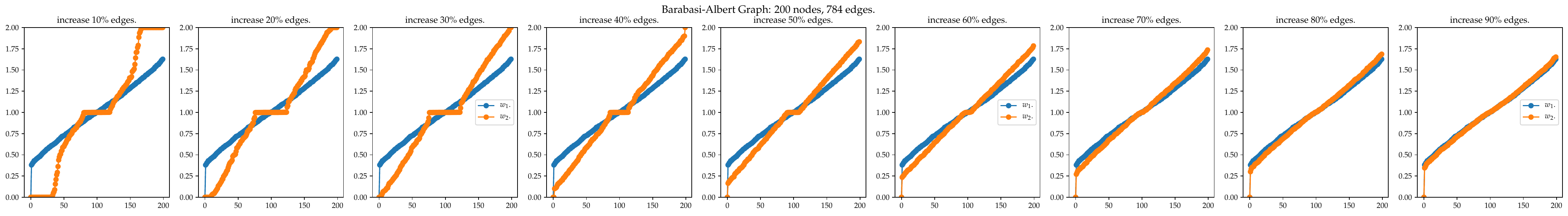}} 
\centerline{\includegraphics[width=\textwidth]{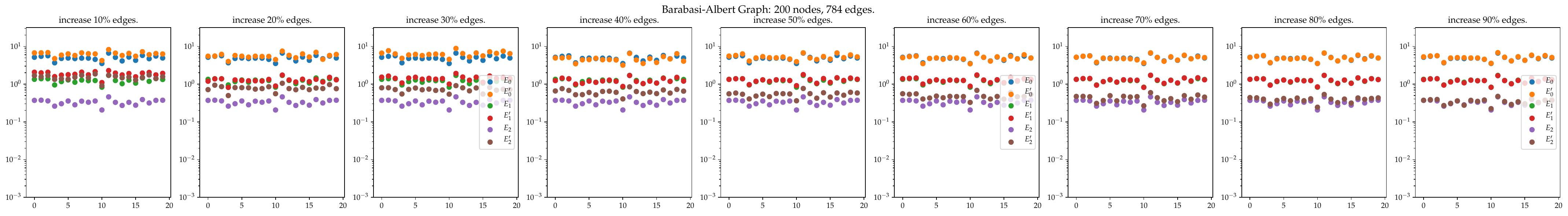}} 
 Low Eigenvector Mix + Drop Edge  
\ 
\centerline{\includegraphics[width=\textwidth]{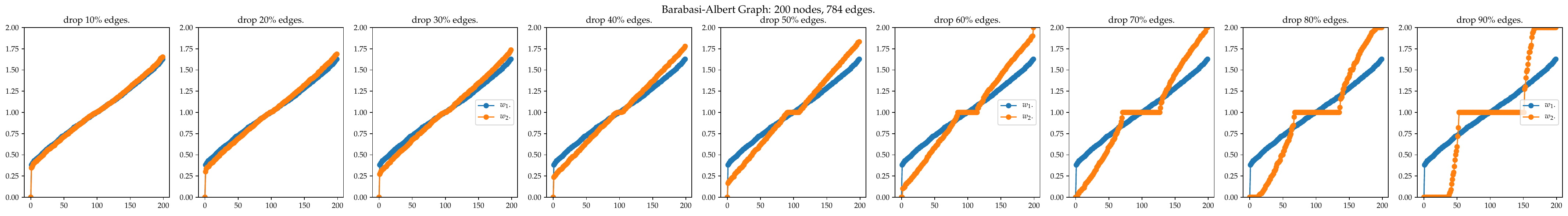}} 
\centerline{\includegraphics[width=\textwidth]{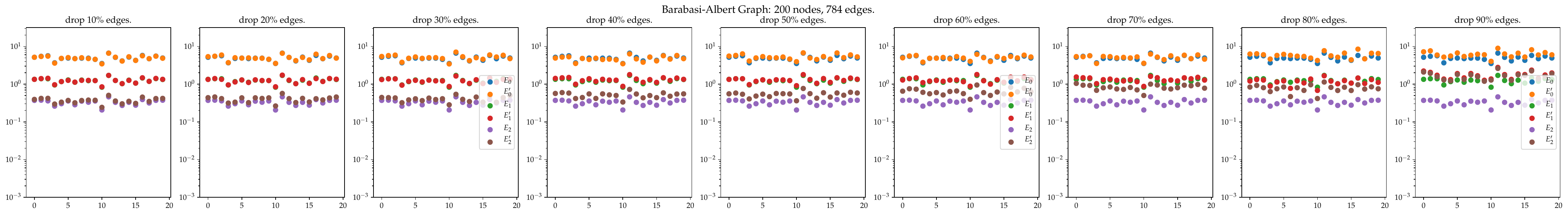}} 
\caption{Barabasi-Albert Graph. } 
\end{center} 
\vskip -0.2in 
\end{figure*}

\newpage
\bibliography{ref.bib}
\bibliographystyle{icml2020}
\end{document}